\documentclass{article}

\usepackage[preprint, nonatbib]{neurips_2020}

\usepackage[utf8]{inputenc} % allow utf-8 input
\usepackage[T1]{fontenc}    % use 8-bit T1 fonts
\usepackage{hyperref}       % hyperlinks
\usepackage{url}            % simple URL typesetting
\usepackage{booktabs}       % professional-quality tables
\usepackage{amsfonts}       % blackboard math symbols
\usepackage{nicefrac}       % compact symbols for 1/2, etc.
\usepackage{microtype}      % microtypography

\usepackage{amsmath, amsthm, amssymb}
\usepackage{todonotes}
\usepackage{float, graphicx, wrapfig}
\usepackage{booktabs, array, caption, subcaption, multirow}
\newcommand{\FF}{\mathcal{F}}

\newcommand{\RR}{\mathbb{R}}

\newcommand{\argmax}{\mathsf{argmax \;}}
\newcommand{\argmin}{\mathsf{argmin \;}}

\newtheorem{theorem}{Theorem}
\newtheorem{lemma}{Lemma}
\newtheorem{corollary}{Corollary}

\usepackage{chngcntr, apptools}
\AtAppendix{\counterwithin{lemma}{section}}
\AtAppendix{\counterwithin{theorem}{section}}
\AtAppendix{\counterwithin{corollary}{section}}

\newcommand{\eps}{\epsilon}
\newcommand{\pr}[1]{\Pr\left[#1\right]}
\newcommand{\prOther}[2]{\mathop{\Pr}_{#1}\left[#2\right]}
\newcommand{\E}[2]{\mathop{\mathbb{E}}_{#1}\left[#2\right]}
\newcommand{\Var}[1]{\mathop{Var}\left[#1\right]}

\newcommand{\ip}[1]{\langle#1\rangle}
\newcommand{\norm}[1]{\|#1\|}

\title{Attention improves concentration when learning node embeddings}

\author{%
  Matthew Dippel\\
  Khoury College of Computer Science\\
  Northeastern University\\
  Boston, MA \\
  \texttt{mdippel@ccs.neu.edu}\\
  \And
  Adam Kiezun \\
  Amazon\\
  Boston, MA \\
  \texttt{akkiezun@amazon.com} \\
  \And
  \hspace{2em}Tanay Mehta \\
  \hspace{2em}Khoury College of Computer Science\\
  \hspace{2em}Northeastern University\\
  \hspace{2em}Boston, MA \\
  \hspace{2em}\texttt{mehta.ta@northeastern.edu}\\
  \And
  Ravi Sundaram \\
  Khoury College of Computer Science\\
  Northeastern University\\
  Boston, MA \\
  \texttt{r.sundaram@northeastern.edu} \\
  \And
  \hspace{5em}Srikanth Thirumalai \\
  \hspace{5em}Amazon\\
  \hspace{5em}Palo Alto, CA\\
  \hspace{5em}\texttt{srikantt@amazon.com} \\
  \And
  \hspace{3em}Akshar Varma \\
  \hspace{3em}Khoury College of Computer Science\\
  \hspace{3em}Northeastern University\\
  \hspace{3em}Boston, MA \\
  \hspace{3em}\texttt{varma.ak@northeastern.edu} \\
}

\begin{document}

\maketitle
\begin{abstract}
We consider the problem of predicting edges in a graph from node attributes in an e-commerce setting. Specifically, given nodes labelled with search query text, we want to predict links to related queries that share products. Experiments with a range of deep neural architectures show that simple feedforward networks with an attention mechanism perform best for learning embeddings. The simplicity of these models allows us to explain the performance of attention.

We propose an analytically tractable model of query generation, AttEST, that views both products and the query text as vectors embedded in a latent space. We prove (and empirically validate) that the point-wise mutual information (PMI) matrix of the AttEST query text embeddings displays a low-rank behavior analogous to that observed in word embeddings. This low-rank property allows us to derive a loss function that maximizes the mutual information between related queries which is used to train an attention network to learn query embeddings. This AttEST network beats traditional memory-based LSTM architectures by over 20\% on F-1 score. We justify this out-performance by showing that the weights from the attention mechanism correlate strongly with the weights of the best linear unbiased estimator (BLUE) for the product vectors, and conclude that attention plays an important role in variance reduction.
\end{abstract}

\section{Introduction}
Graphs are used in various applications such as bioinformatics~\cite{fout2017protein}, recommender systems~\cite{ying2018graph}, social network analysis~\cite{Hamilton17Graphsage}, etc. An important learning problem on these graphs is to predict whether two nodes have an edge given information about other nodes in the graph. Solving this is crucial for tasks like metabolic network construction, movie recommendations, and knowledge graph completion.

One approach for solving such downstream tasks like link prediction is to employ low-dimensional vector representations of nodes generated by latent variable models. These techniques were originally developed for representing text and separately for representing images. Subsequently, these techniques have been combined~\cite{Hamilton17Graphsage} to generate embeddings for graphs that respect the semantic similarity of node features (text/images).

In this work, we address the link prediction problem in an e-commerce setting. Here we consider the \emph{query graph} consisting of nodes representing search queries entered by users seeking a specific product. There exists a link between two queries if users purchased the same product after making each query. The \emph{query reformulation problem} is to infer the links in the graph for a newly added node labeled by its query. For example, a user may enter the never-before seen query ``anxiety toy''. The system should infer that the user is searching for products also bought by searches for ``fidget spinners'', and consequently, there must be a link between the two queries in the query graph. It is easy to see from the example that purely syntactic (string matching) approaches are insufficient. Furthermore, note that queries cannot be considered as a bag-of-words; the sequence of words conveys important meaning. For instance, the queries ``milk chocolate'' and ``chocolate milk'' contain the same words but mean different products.

Existing latent variable approaches are unable to solve this problem (see Related work)). Keeping the above examples in mind, we approached the query reformulation problem using deep latent variable models that are sensitive to word sequences. In particular, we considered long short-term memory networks (LSTMs) as well as feedforward networks with attention. Indeed, more sophisticated models like BERT \cite{devlin-etal-2019-bert} require billions of parameters, which make them infeasible for training due to the large vocabulary size of commercial datasets. LSTMs and in particular the feedforward network with attention require fewer parameters and resources to train. Our experiments (see section~\ref{sec:experiments}) showed that attention networks perform very well and are significantly quicker to train.

This raises the question of why attention networks perform well. Indeed, there has been existing work investigating the limits of attention networks (see Related work). However, we are not aware of a sound analytical justification for the success of attention networks. Our main contribution is to offer a succinct, model-based explanation that may serve as a basis toward understanding more sophisticated models such as BERT.

\textbf{Our results}
We solve the query reformulation problem using a feed forward attention network with a cross-entropy-based loss function. For the purpose of judging the output of the models, we come up with a novel analog of the F1 score intended to capture both relevance of outputted queries as well as the diversity of the products that the reformulated queries lead to. We compare the attention-based approach to a hybrid method using graph embeddings and a long short-term memory network (LSTM). We also compare to a pure text based approach. We show that the attention mechanism beats (all reasonable variants of) these other approaches by over 20~\% on the F-1 score (Section~\ref{sec:experiments}).

We formulate a model for query generation, Attention Embeddings for Short Texts (AttEST), that matches statistical properties of queries and allows us to explain the success of attention networks with the cross-entropy loss function. Analogous to word embeddings, we show that the PMI of two queries is the dot product of their vector embeddings (see Corollary~\ref{cor:pmi-dot-product}), matching the empirical observation that the query PMI matrix is low rank. Using this property we give theoretical validation for the cross-entropy-based loss function as maximizing mutual information. The AttEST model also allows us to prove that a weighted average of trigram vectors is the best linear unbiased estimator (BLUE) for the product desired by a query (in Section~\ref{sec:attention}). Interestingly, in  Figure~\ref{fig:attention-weights-vs-BLUE-weights}, we observe a notable correlation between the (empirical) weights from the attention mechanism and the BLUE weights derived in Corollary~\ref{lem:blue}. This validates the AttEST model and suggests that the attention mechanism weights allow more efficient concentration to the product vector by reducing the variance in the estimation.

\textbf{Related work} \label{sec:relatedwork}
Since the success of word2vec~\cite{MikolovCCD13word2vec} in finding word embeddings, there have been many variants and extensions~\cite{Pennington14glove, de16shorttext, arora17sentence, doc2vec-para2vec}. There has also been a lot of work on embeddings for nodes in a graph~\cite{GCN-kipf,Hamilton17Graphsage,Narayanan17graph2vec,velivckovic2017-graph-attention-networks-GAT}. The former set of works do not consider any graph structure and only have embeddings for textual features while the latter set of works do not have features associated with the nodes and hence cannot incorporate that information to generate embeddings.

Another line of work in representation learning leverages multiple types of entities, for example text and images~\cite{deep-image-representations,image-embeddings-cnn-2014}, to make multi-modal models to get embeddings for both entities in the same latent space. \cite{arora-contrastive-learning} provides a theoretical framework for generating embeddings for entities that have a semantic similarity relationship between them, generalizing embeddings for nodes in a graph. However, their framework does not account for features associated with entities and hence their framework does not apply in our scenario.

\cite{vaswani17attention} showed that the attention mechanism of~\cite{bahdanau14attention} supplants and outperforms recurrent models for many problems via the Transformer network. We indirectly corroborate this by showing that the weighting scheme induced by an attention mechanism gives the least variance estimator for the true embedding. On the other hand, \cite{jain19attention} argue against using attention weights as a measure of feature importance for RNN-based models. This does not contradict our reasoning of attention weights enhancing query embeddings since we employ simple, feedforward networks with attention. \cite{devlin-etal-2019-bert} introduces the BERT model which uses bidirectional training of the Transformer for training language models. BERT and its augmentations \cite{roberta, qbert} represent the state-of-the-art in language modeling.

Our work builds, in nontrivial fashion, on the seminal RAND-WALK model of \cite{arora16randwalk}. RAND-WALK is a generative model of word embeddings, that provides an explanation for the low-rank nature of the point-wise mutual information (PMI) matrix~\cite{Deerwester90indexingby,Turney10frequency}, among others. In contrast to RAND-WALK which analyzed long form text with small vocabularies and could exploit ergodicity, AttEST analyzes short queries over a massive vocabulary and so required novel yet justifiable modeling assumptions; in addition to corroborating the low-rank nature of PMI matrices, AttEst explains the effectiveness of both the loss function and the attention mechanism, with empirical validation.

%For the particular concrete problem we consider, there have been works using neural network models to address it~\cite{wu17shopping}.

In the specific context of e-commerce there have been works conducting an empirical study using LSTM networks to map queries to structured attributes~\cite{wu17shopping}, as well as works that consider the more specific problem of ranking query reformulations~\cite{sheldon2011lambdamerge,santos2013learning}. As opposed to the former work, our latent space model AttEST allows for arbitrary downstream tasks on queries while having a theoretical grounding. This theoretical grounding also solves the reformulation ranking problem by following the embedding step with a $k$-nearest neighbor search in the latent space to shortlist reformulations.

\section{Experimental results} \label{sec:experiments}
Our primary data set uses query-product pairs from the Electronics category of the Amazon US locale sampled during the March - April 2018 period. The resulting query-product graph has approximately 670,000 queries, 146,000 products, and 1 million edges. The second data set is sampled from the Amazon US locale over a period of 91 days, up to August 26, 2017. We partition the queries and products into disjoint clusters using a spectral clustering algorithm. From these, we take all queries and products which appear in the largest 25 clusters. In total, the data set is approximately 250,000 queries. From the bipartite query-product graph, we completed all triangles and took the resulting query-query graph on only the query nodes. Finally, we removed isolated queries from the query-query graph, and viewed the edges of the graph as data samples. Both the primary and Top-25 Clusters data sets are partitioned as 95\%--5\% training--testing split.
%\footnote{This is so as not to create spurious isolated query nodes.} 

\subsection{Metrics for query reformulation}\label{sec:metrics-for-query-reformulation}
Evaluating the quality of query reformulations is a non-trivial task since the data set only contains the products that were \emph{purchased} by the customer searching for a given query. Because the product search engine returns more than twenty products on the first page alone, all of which may be relevant, it is not clear how to determine whether the nearest-neighbor queries are useful reformulations. We propose two metrics which measure the precision and recall of top five reformulated queries.
\vspace{-0.6em}
\begin{enumerate}
\item \textbf{Query Precision@K}: Precision is defined to be the fraction of the five reformulations $q_i$ which are `relevant' to the initial query $q$. We say a reformulation $q_i$ is `relevant' if the top $K$ products (by purchases) of $q_i$ contain at least one of the top $K$ products of $q$.
\item \textbf{Product Recall@K}:  Recall is defined to be the fraction of the top K products associated with $q$ that appear in the list of top K products associated with \emph{some} $q_i$.
\end{enumerate}
\vspace{-0.6em}

Query precision measures the fraction of reformulated queries that are `valid', while product recall measure the diversity of reformulations.
% The last statement follows from the sparsity of the dataset.
Since the associated products of a query are only derived from clicks, adds to shopping cart, and purchases, very similar queries are likely to have no overlap of products. Therefore, a high product recall score suggests that the list of reformulations is diverse. Since these metrics are ultimately evaluated from purchase behavior associated with queries, they lower bound the `true' precision and recall as would be determined by a human evaluator.

\vspace{-0.5em}
\subsection{Models}
\vspace{-0.3em}
%In this section, we describe the various models used in experiments for the query reformulation problem.
\textbf{Trigram Hash}
As a baseline, we consider the purely textual \textbf{Trigram Hash} model - it ignores any behavioral connections. In particular it would not associate "fidget-spinner" with "anxiety attention toy" since the two strings are textually so dissimilar. Each query is treated as a bag of trigrams and hashed down to a 300-dimensional vector. And given a new (test) query the nearest neighbors algorithm is used to find the closest training queries based on the Bray-Curtis distance.

% More formally, from the set of \textasciitilde 30k unique trigrams, we generate a frequency vector for each query and then hash that vector down to 300 dimensions. We use a modulus based hashing, with index $i$ in the frequency vector being hashed to the index $i\% 300$. Given the 300 dimensional representation of the test query $U$, we find the training dataset query $V$ which has the lowest Bray-Curtis distance to it: $\sum_i(|U_i - V_i|)/\sum_i(|U_i| + |V_i|)$. This ratio captures the size of the intersection to the size of the union of the two queries.

\textbf{LSTM to match graph embeddings}
There is plenty of existing literature on generating meaningful embeddings for nodes in a graph as we detailed in the Related work section. We make use of two such tools, node2vec~\cite{Grover16node2vec} and GraphSAGE~\cite{Hamilton17Graphsage}, by applying them to the Query-Product graph. We take the learned embeddings for query nodes, and train an LSTM to match the embeddings given only the query string. Matching an LSTM to either of the embeddings performs similarly; we present the \textbf{LSTM+node2vec} results which were a few percentage points better.

We also have a \textbf{LSTM only} model which directly gets the text as input. We enforce the graph structure in a query embedding via positive samples, queries which share the same product set, and negative samples, queries with no products in common. Let $z_i$ be a proposed embedding for a query, $Z_p$ a set of positive sample embeddings, and $Z_n$ a set of negative sample embeddings. The following loss function will then drive $\sigma(z_i^{T}z_p)$ to $1$ and $\sigma(z_i^{T}z_n)$ to $0$.
\begin{align*}
        \frac{1}{|Z_p|}\sum_{z_p \in Z_p}{-\log\left(\sigma(z_i^{T}z_p) \right)} + \frac{1}{|Z_n|}\sum_{z_n \in Z_n}{-\log\left(\sigma(- z_i^{T}z_n) \right)}
\end{align*}

\textbf{Attention on textual input}
Our main model simply learns trigram embeddings, and uses attention on the trigram embeddings to compute a weighted average. Note that the learning is inductive and unsupervised with the loss function agnostic to the downstream metric (in particular to the diversity component, Product Recall@K).

The input to the network is a vector with length equal to the maximum query length of $50$. Here each coordinate of the input has a unique number identifying the trigram. The network then uses an embedding layer to come up with a vector representation $v_t$ of each trigram. The attention mechanism is simply applying a different linear transformation $\mathbf{W_i}$ to each vector, getting $\mathbf{W_i}v_{t_i}$ and then taking softmax of these to get corresponding weights. The final query vector outputted is the weighted average of the trigram vectors using the attention weights.

While we always use the same loss function, we have a few variants in our experiments. First, we consider two positive sampling methods while training: uniformly sampling from neighbors, and sampling using the GraphSage approach of running multiple, fixed-length random walks from each node, and using all co-occurring pairs of nodes as positive samples. Since the GraphSage sampling performs slightly better, we only present those results in Table~\ref{tab:results}. The second variant additionally provides word data on top of trigram data. This method performs slightly worse than providing only the trigrams which we believe is due to the trigram-only model being better equipped to handle typographical errors. The word model wastes some attention weights on the words, which may have such errors thus adversely affecting performance. % The final variant is to make the Attention model try to learn pretrained GraphSage embeddings (for the query-query graph), analogous to the LSTM variant.

\subsection{Results}

\renewcommand\arraystretch{1.25}
\begin{table*}[h]
\begin{center}
\caption{Experimental results of the various described models. Query Precision and Product Recall are percentages of the best possible scores on the test dataset found by brute force search.}\label{tab:results}
\begin{tabular}{clrrrr}
\toprule
  & \textbf{Model} & \textbf{Query Precision@20} & \textbf{Product Recall@20} & \textbf{F1}\\ \midrule
  \multirow{4}{*}{Top-25 Clusters} & Trigram Hash & 45.6\% & 60.3\% & 47.6\%\\
  & LSTM only & 43.2\% & 55.9\% & 44.4\% \\
  & LSTM+node2vec & 59.8\% & 64\% & 57.2 \% \\
  & Attention & \textbf{65.9\%} & \textbf{70.8\%} & \textbf{65.3\%}\\
  \midrule
\multirow{4}{*}{Electronics} & Trigram Hash          & 37.22\% & 51.85\% & 41.62\%\\
& Attention+Word     & 52.22\% & 61.41\% & 55.02\%\\
&  Attention & \textbf{59.44\%} & \textbf{68.20\%} & \textbf{62.20\%}\\
\bottomrule
\end{tabular}
\end{center}
\end{table*}

For the Top-25 clusters dataset, the LSTM models performed reasonably comparably to the attention models. However, when compared on the primary dataset, the LSTM models performed very poorly. We conjecture that the LSTM models were memorizing textual information for the top queries and since the Top-25 clusters dataset only had 25 popular clusters, the LSTM model was able to predict the correct product class fairly easily. This strategy became useless in the larger primary dataset due to infeasible training times for LSTMs. The larger and sparser dataset allows the attention model to really shine through.

% \footnote{We did not increase the size of the models as training the LSTM took too much time to be feasibly deployed.}

The attention models performed much better than any of the other neural models. Given that the attention model and loss function are agnostic to the metric one reason that the Top-25 Clusters performance is superior to the Electronics performance could be that clustering enhances the diversity component of the F-1 score. Interestingly, the baseline Trigram Hash model was the runner-up.

%We observed that the above metric was highly bimodal. Most of the test queries either had $0$ or $1$ precision and recall. Since the test queries have varying number of associated products, a better augmentation of the above metric would be to take the weighted average of the precision, recall and F1 scores. If we wish to consider common queries to be more important, we can weight the scores with the number of products associated with each test query. If we wish to consider tail queries to be more important, we can weight the score with the inverse of the number of products.

\section{AttEST: Model and Theoretical Results}
In the model, query generation is viewed as a two step process where the user first thinks of a product to search for, and then generates a query based on that product. Let $S^d = \{v \in \RR^d ~|~ \norm{v} = 1 \}$ be the $d$-dimensional unit sphere. A product $p$ is selected by sampling a vector uniformly from $S^d$. A query $q$ is generated by synthesizing an ordered sequence of n-grams, where the sequence length, $n(q)$, is determined by sampling from a Poisson distribution truncated at the maximum query length $N$. E-commerce vocabulary is in the order of tens of billions, including the regular English lexicon, as well as brands, models, ISBN codes, product codes, etc. The concatenative morphology of English (e.g. `antigovernment' is sum of the morphemes `anti', `govern', and `ment') and various codes (e.g. ISBN) allows us to derive meaning from the constituent n-grams. We use trigrams for our experiments due to the memory requirements of our data-set prohibiting larger n-grams. Let $T = \{t \in \RR^d\}$ be an isotropic set of $m$ vectors representing all possible trigrams. The $i$-th trigram $t_i = t$ of $q$ is sampled from this mixture distribution on $T$,
\begin{align}
    P_{p, i}(t) = \alpha_i\cdot \frac{\exp(\beta_i*\ip{t, p})}{Z_{p, i}} + (1-\alpha_i)\cdot \frac{1}{m}\label{eq:mixture-distribution}
\end{align}
where $Z_{p, i} = \sum_{t \in T}\exp(\beta_i*\ip{t, p})$ is the partition function for the exponential distribution in the mixture, $\alpha_j \in (1/2,1]$ is the mixture parameter and $\beta_j$ is the positional spread parameter. In a mild abuse of terminology we will use trigram and product to refer their corresponding vectors.

The exponential component of the mixture samples trigrams near the product while the uniform component models noise in the generative process. We follow the log-linear model introduced by \cite{arora16randwalk} but differ from it in several key ways to accommodate specific characteristics of (short text) queries that are not found in (longer form) written language. Our changes help model query generation in a natural manner. A user searches for a product by listing the attributes associated with it, and while most trigrams would be very relevant to the product, there will inevitably be some that introduce noise into the query. The position dependent mixing and spread parameters control how the noise changes depending on where in the query the trigram is. Generally, it becomes noisier as the query becomes longer. We make the simplifying assumption that the trigrams are all sampled independently of each other, which is not true in practice.

To generate the graph, queries are first generated according to the mixture in equation \ref{eq:mixture-distribution}. Each query $q$ has an associated vertex labeled by its trigrams $t_1, \dots t_{n(q)}$. Two queries $q, q'$ are adjacent in the graph if they were generated by product vectors $p, p'$ such that $\norm{p - p'}_2 \leq \eps_p$ for some parameter $\eps_p$.

\subsection{Attention}\label{sec:attention}
We state some basic properties of the mean, variance and partition function of the trigrams sampled by the AttEST model (Equation \ref{eq:mixture-distribution}).

\begin{lemma}\label{lem:expectation-of-trigram}
  Let $\rho_i = m\alpha_i\beta_i\exp(\beta_i^2/2)/Z_{p, i}$. The mean of the trigram $t_i$ in the $i$-th position is
  \begin{align*}
  \E{T \sim \mathcal{N}(0, I)^m}{\E{t_i \sim P_{p,i}}{t_i}} = \rho_i \cdot p.
  \end{align*}
\end{lemma}
 %Note that, due to linearity of expectation, the expected value of any (convex) combination of trigrams will be aligned along the product. Also, the variance of a trigram is a function of its position.
 \begin{lemma}\label{lem:variance-of-trigram}
The expected $\ell_2$-distance squared of $t_i$ from its mean is $\Theta(d \ln d) - \rho_i^2$
\end{lemma}
For large vocabularies ($m \rightarrow \infty$), the partition function $Z_{p, i}$ can be approximated by a constant $Z_i$. 
\begin{lemma}[Concentration of partition functions, Lemma 2.1 from~\cite{arora16randwalk}]
  \label{lem:concentration-of-Z}
  For trigram vectors of the form $v_t = s_t \cdot \hat{v}_t$, where $\hat{v}_t$ comes from a spherical Gaussian distribution, and for $\eps_z = \widetilde{O}(1/\sqrt{m})$ and $\delta = \exp(-\Omega(\log^2 m))$ there exists $Z_i$ s. t.,
    $\prOther{p\sim S}{(1-\eps_z)Z_i \leq Z_{p, i} \leq (1+\eps_z) Z_i} \geq 1 - \delta$.
\end{lemma}
% The proof for this lemma follows from the proof of Lemma 2.1 from~\cite{arora16randwalk}.

\subsection{Low rank of PMI}
\begin{theorem}\label{thm:co-occurrence-probability}
Let $q, q'$ be query vectors generated by the AttEST model. Denote the probability that $q, q'$ co-occur in the query graph by $\pr{q, q'}$. Then,
\begin{align*}
    \pr{q, q'} = & ~ (1\pm \eps') \cdot \Bigg(\prod_{i \in [n(q)]} \frac{\alpha_{i}}{Z_i}\prod_{j \in [n(q')]} \frac{\alpha_{j}}{Z_j} \Bigg) \cdot \exp\Bigg(\frac{\norm{\sum_{i\in[n(q)]}\beta_i t_i + \sum_{j\in[n(q')]}\beta_j t_j}^2}{2d}\Bigg)
\end{align*}
\end{theorem}
\begin{proof}[Proof sketch]
Start by averaging the event that $q, q'$ co-occur over all product vectors $p, p'$ that are close enough to have an edge in the query graph.
\begin{align*}
  \pr{q, q'} & = \E{p, p' \sim S}{\pr{q, q'} \;~\big|~\; \|p - p'\| \leq \eps_p/d)} = \E{p \sim S}{\pr{q, q'} \;~\big|~\; \|p - p'\| \leq \eps_p/d)}\\
             & = \E{p \sim S}{\pr{q~|~p} \pr{q' ~|~ p'} } = \E{p \sim S}{\prod_{i \in [n(q)]} \pr{t_i ~|~ p} \prod_{j \in [n(q')]} \pr{t'_j ~|~ p'}}
\end{align*}
Note that the two products in the last line take probabilities from the mixture distribution. In order to complete the proof, we take the following steps. First, we use Lemma \ref{lem:concentration-of-Z} to factor out the partition functions from the equation. Next, we show that the uniform component from the mixture can be ignored without incurring too much error. Finally, we remove the dependence on the second product $p'$ by exploiting its closeness to $p$ on the assumption that co-occurring queries are generated by nearby products. These results allows us to complete the calculation and finish the proof.
\end{proof}

\begin{theorem}\label{thm:single-co-occurrence-probability}
Let $q$ be a query vector generated by the AttEST model. Then,
\begin{align*}
      \pr{q} &= (1\pm \eps'') \cdot \Bigg(\prod_{i \in [n(q)]} \frac{\alpha_{i}}{Z_i}\Bigg)\cdot \exp\Bigg(\frac{\norm{\sum_{i\in[n(q)]}\beta_i t_i}^2}{2d}\Bigg)\\
\end{align*}
\end{theorem}
The proof of Theorem~\ref{thm:single-co-occurrence-probability} follows a similar argument to Theorem~\ref{thm:co-occurrence-probability} and together they imply:
\begin{corollary}\label{cor:pmi-dot-product}
For $\eps''' = O(\eps' + \eps'')$,
$\textrm{PMI}(q, q') \triangleq \log \Bigg( \frac{\Pr[q, q']}{\Pr[q]\Pr[q']} \Bigg) = \frac{\ip{q, q'}}{d} + \eps'''.$
\end{corollary}
Since the query vectors $q, q'$ are $d$-dimensional, the corollary shows that the PMI is rank $d$. 
\subsection{Loss function derivation}
We now provide theoretical justification for the cross-entropy loss used in the AttEST attention model.
Let $Q$ be the set of all empirical queries, let $P(q)$ ($N(q)$) denote the queries (not) adjacent to $q \in Q$ in the query-query graph.
\begin{align*}
L = \argmin \sum_{q \in Q, q' \in P(q)} -\log(\sigma(\langle v_q, v_{q'} \rangle)) + \sum_{q \in Q, q' \in N(q)} -\log(\sigma(-\langle v_q, v_{q'} \rangle))
\end{align*}

To derive the loss function, we start by maximizing the mutual information (MI) of the marginal distributions of each endpoint of the edge distribution $E$ while minimizing the MI between the marginal distribution of each endpoint of the non-edge distribution $\bar{E}$. This ensures that the amount of information derived from a query embedding about its related queries (and only its related queries) is maximized. Let $E_1, E_2$ be the marginal distributions for the first and second vertex in the edge sampled from $E$. Similarly, let $\bar{E}_1, \bar{E}_2$ be the marginal endpoint distributions for non-edges sampled from $\bar{E}$. Recall the definition of mutual information, 
\begin{align*}
I(E_1; E_2) \triangleq & \sum_{q, q' \in Q} \Pr[q, q']~ \textrm{PMI}(q, q').
\end{align*}
We maximize the mutual information between $E_1$ and $E_2$ while minimizing it between $\bar{E}_1$ and $\bar{E}_2$.
\begin{align*}
    \argmax I(E_1; E_2) - I(\bar{E}_1; \bar{E}_2)
\end{align*}
We can equivalently maximize the exponential of the above.
\begin{align*}
    =& ~ \argmax ~\exp(I(E_1; E_2) - I(\bar{E}_1; \bar{E}_2))\\
    =& ~ \argmax ~\prod_{q, q' \in Q} \exp(\pr{q, q'}\textrm{PMI}(q, q') - \pr{q, q' \text{ not adjacent}]}\textrm{PMI}(q, q'))
    \intertext{After some algebraic manipulation and approximations (see Appendix), we obtain the following:}
    \approx& ~ \argmax ~\prod_{q \in Q, q' \in P(q)}\exp (\textrm{PMI}(q, q')) \prod_{q \in Q, q' \in N(q)} \exp(-\textrm{PMI}(q, q')) \\
    %=& \argmax ~\exp \left(\sum_{q \in Q, q' \in P(q)} \textrm{PMI}(q, q')\right)\\
     =& ~ \argmax ~\exp \left(\sum_{q \in Q, q' \in P(q)} \ip{v_q, v_{q'}} ~+ \sum_{q \in Q, q' \in N(q)} -\ip{v_q, v_{q'}}\right)
\end{align*}
 where the last line follows from Corollary \ref{cor:pmi-dot-product}. We can remove the exponential due to monotonicity. Note that in a small range around $0$, the sigmoid function may be approximated by an exponential allowing us to take the logarithm of the sigmoid for each term in the sums, and get the loss function.
 %In order to avoid exploding and vanishing gradients, we apply the logarithm of logistic function to smooth out the value of the inner product. Heuristically, applying a logarithm leaves the dynamics of local search algorithms (gradient descent) largely unchanged \cite{SV14}.
\vspace{-0.6em}
\begin{align*}\argmax ~\sum_{q \in Q, q' \in P(q)} \log (\sigma ( \ip{v_q, v_{q'}})) ~~~ +  \sum_{q \in Q, q' \in N(q)} -\log(\sigma(-\langle v_q, v_{q'} \rangle)) \end{align*}
% Thus, maximizing MI for adjacent queries and minimizing it for non-adjacent queries gives the loss function .

\subsection{Experimental validation via trigram variance}
Intrigued by the success of the attention mechanism we ran several statistical analyses on the attention weights. An immediate observation was that the curve of attention weights by trigram position exhibited a downward trend.  This matches the intuition that the importance or information content of the trigrams early in the query is high while those at the end of a long query are less useful for inferring the product the user has in mind. The initial oscillations also seemed to suggest that users of search engines tend to order their descriptors from more important to less (e.g. `iphone white 32gb' is preferred over `32gb white iphone'). However, this qualitative link between the semantics of search and empirical attention weights do not immediately suggest a quantitative link to our theoretical AttEST model. Our first inkling of an explanation for this success came when we noticed that the sequence of attention weights correlated strongly with the inverse of their variance which we explain now.

Let $X_1, \dots, X_k$ be independent, real-valued random variables drawn from different distributions, such that all distributions have the same expectation $\mu$ and (possibly different) variances $\sigma_1, \dots, \sigma_k$, respectively. A special case of the well-known Gauss-Markov theorem states that:
\begin{lemma} \label{lem:blue}
The best linear unbiased estimator (BLUE) of $\mu$ is $\sum_{i=1}^k w_iX_i$ where $w_i = \frac{1/\sigma_i^2}{\sum_{i=1}^k 1/\sigma_i^2}$.
\end{lemma}

When we plot the empirical attention weights and the BLUE weights, calculated using Lemma~\ref{lem:blue} but based on the empirical variance of the attention weights we see a remarkably good fit (see Figure~\ref{fig:left}). This led us to the hypothesis that attention weights essentially function as BLUE weights enabling more accurate inference of the product (vector) from the trigrams. 
\iffalse
Looking more closely at the attention weights curve, we see that ignoring the initial oscillations, the weights exhibit an overall downward trend. Lemma \ref{lem:blue} implies that the variance of the query vector has a global upward trend. Indeed, this matches the intuition that the trigrams at the end of a long query are not useful for search. 
\fi

From Lemma~\ref{lem:expectation-of-trigram} we see that the mean of the trigram vectors (suitably scaled) is an unbiased estimator of the product but how exactly do the variances predicted by the AttEST model manage to fit the empirical variances? Recall from Equation~\ref{eq:mixture-distribution} that we have two parameters - the noise parameter $\alpha_i$, and the spread parameter $\beta_i$ - for each trigram position $i$, i.e., the AttEST model has two degrees of freedom to perfectly fit the variance value at position $i$. To validate the observed fit of the AttEST model we fixed all the $\beta_i$ to a constant and assumed a simple linear form for the variances, i.e., $\sigma_i^2 \propto i$. By Lemma~\ref{lem:variance-of-trigram} this implies a simple fixed form for the $\alpha_i$. The linear growth of variance with trigram position is consistent with the observation that queries tend to get more discursive and ramble the longer they go on. It is also consistent with the practice of FKMR (fewer keywords, more results) enshrined in search engines where later parts of queries are preferentially dropped in order to provide meaningful results. Figure~\ref{fig:middle} shows an excellent fit between the empirical attention variance (scatter plot) and the theoretical AttEST variance (idealized line) using this linear model of growth in variance with trigram position. 

Finally, in Figure~\ref{fig:right} we plot the values of the $\beta_i$ terms so as to perfectly match up the empirical attention variance with the idealized variance line; the $\beta_i$ values can be thought of as the residuals. Recall that the $\beta_i$ parameter captures how strongly the chosen trigram at that position aligns with the product vector. In the figure, we notice the initial high fluctuation with a peak roughly between positions 10 and 20 which corroborates the intuition that the most relevant keywords can be seen once the first word has narrowed the category and the second word focuses the query onto the product. Further, looking at the tail end we see that the importance of $\beta_i$ goes down and it becomes a flat curve indicating that the later keywords are not as relevant to determining the product.

Given that the attention weights behave as the BLUE weights we now provide a quantitative justification for why the weighted average out-performs the unweighted average in terms of the variance of the inferred product (vector).  
\begin{lemma}\label{lem:query-variance}
Let the variances $\sigma_i^2$ of each trigram be proportional to $i$, Then the variance of the unweighted query vector is $\Omega(1)$ while the variance of the weighted query vector is $o(1)$.
\end{lemma}
\vspace{-1em}
\begin{proof}
\vspace{-0.5em}
  \begin{enumerate}
  \item The variance of the unweighted query vector is:
      $\Var{\frac{\sum_{i=1}^k t_i}{k}} = \frac{\sum_{i=1}^k\sigma_i^2}{k^2}$
  \item The variance of the weighted query vector is:
      $\Var{\frac{\sum_{i=1}^k\frac{t_i}{\sigma_i^2}}{\sum_{i=1}^k\frac{1}{\sigma_i^2}}} = \frac{1}{\sum_{i=1}^k \frac{1}{\sigma_i^2}}$
\end{enumerate}

For the unweighted case, set $\sigma_i = \sqrt{i}$ and see that $\Var{\frac{\sum_{i=1}^k t_i}{k}} = \frac{\sum_{i=1}^k\sigma_i^2}{k^2} = \frac{\sum_{i=1}^ki}{k^2} = \frac{k(k+1)}{2k^2} =$ $1/2 + 1/2k = \Omega(1)$. For the weighted case, we get that $\Var{\frac{\sum_{i=1}^k\frac{t_i}{\sigma_i^2}}{\sum_{i=1}^k\frac{1}{\sigma_i^2}}} = \frac{1}{\sum_{i=1}^k \frac{1}{\sigma_i^2}} = \frac{1}{\sum_{i=1}^k \frac{1}{i}} = 1/H_k$ where $H_k$ is the $k^{\textrm{th}}$ Harmonic number, which is approximately $\ln k$. In particular, if $\sigma_i \propto \sqrt{i}$, then the variance is $o(1)$.
\end{proof}
\vspace{-0.6em}

Lemma \ref{lem:query-variance} shows that, with growing query length, the variance with attention weights vanishes to zero whereas the variance of the unweighted (or uniformly weighted) case remains at a constant bounded away from 0. In other words the explanation for the success of attention mechanisms is that they provide an efficient method to reduce the variance (increase concentration) in the estimation of the ground truth.

\begin{figure}[t]
\centering

\hspace{-3ex}
\begin{subfigure}[b]{0.33\textwidth}
\includegraphics[width=\textwidth]{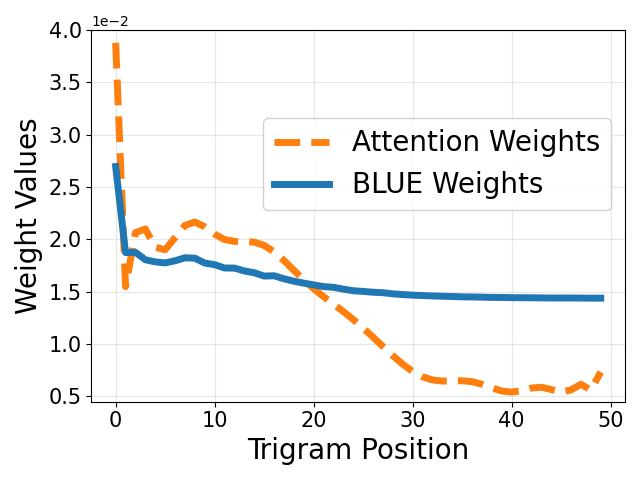}
\subcaption[]{}\label{fig:left}
\end{subfigure}\hspace{-1ex}
\begin{subfigure}[b]{0.33\textwidth}
\includegraphics[width=\textwidth]{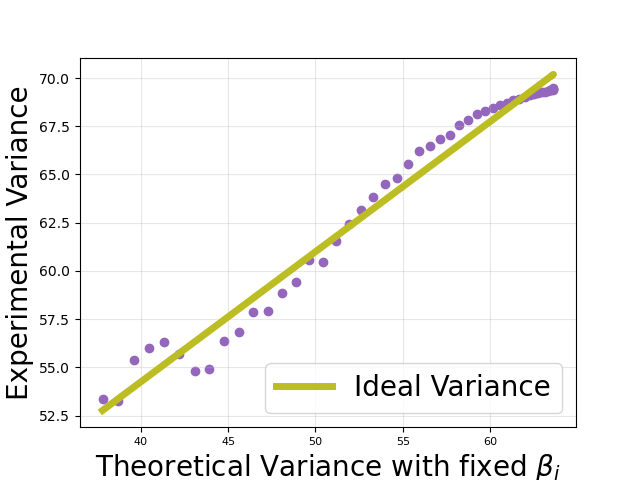}
\subcaption[]{}\label{fig:middle}
\end{subfigure}\hspace{-3ex}
\begin{subfigure}[b]{0.33\textwidth}
\includegraphics[width=\textwidth]{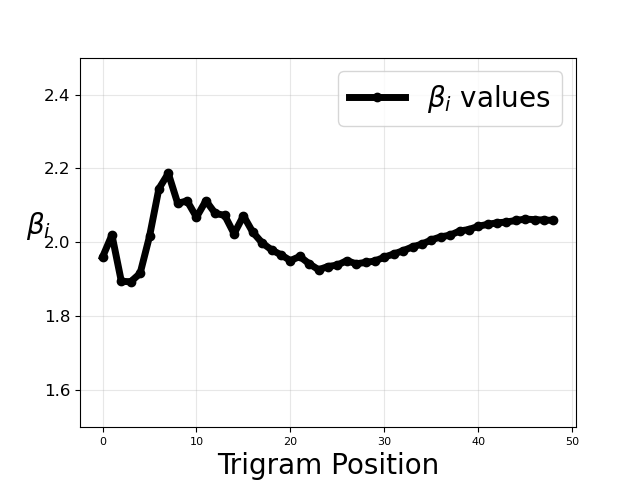}
\subcaption[]{}\label{fig:right}
\end{subfigure}
\caption{From left to right: The left plot is of the weights produced by the attention mechanism, and the BLUE weights predicted by Lemma~\ref{lem:blue}. Weights are averaged over all queries, and shown by the position of the trigram ($i$) in the query. 
% The second plot is blue variance vs. data variance based on a simple functional form for $\alpha_i$. The third is blue variance vs. data variance with adjustments due to $\beta_i$. 
The middle plot is a scatter plot of variance predicted by theory (without changing $\beta_i$ values) alongside the variance from the data. The best linear fit shown by the Ideal Variance curve is achieved by using the $\beta_i$ values shown in the right plot.
% , accounting for $\alpha_i$ but not for $\beta_i$. The last plot is Beta Values to fix theory predicted variance to match linear line
}\label{fig:attention-weights-vs-BLUE-weights}
% \vspace{-1.5em}
\end{figure}

\vspace{-1em}
\section{Conclusion}
\vspace{-1em}
Predicting edges in a graph from node attributes is a general problem beyond e-commerce. Future work can consider alternative datasets with images as node attributes. However, study of alternative datasets requires careful adaptation of the theoretical analysis of this work as the AttEST model relies on the following properties of e-commerce data sets: node attributes are short texts, large vocabulary size, and the correlation of position and weight of trigrams. Analogues for these properties must be identified when analyzing alternative datasets. For instance, would the decomposition of text into trigrams work well in languages without concatenative morphology, such as Arabic or Hebrew? 

\nocite{}
\bibliographystyle{abbrv}
\bibliography{references}

\appendix
\section{Proofs}
Proofs are organized by paper sections.
\subsection{Trigram Properties}
\begin{lemma}\label{appendix:lem:expectation-of-trigram}
  For all $i$, the mean of a trigram $t_i = t$ sampled using the distribution $P_{p,i}$ from the set of all trigrams $T$ is the following, where $\rho_i = m\alpha_i\beta_i\exp(\beta_i^2/2)/Z_{p, i}$.
  \begin{align*}
  \E{T \sim \mathcal{N}(0, I)^m}{\E{t \sim P_p}{t}} = \rho_i \cdot p
  \end{align*}
\end{lemma}
\begin{proof}
 Start by expanding the inner expectation over the $P_p$ distribution,
  \begin{align*}
    & \E{T \sim \mathcal{N}(0, I)^m}{\sum_{t \in T} \left(\alpha_i \cdot \frac{\exp{\beta_i\ip{p, t}}}{Z_{p, i}} + \frac{1-\alpha_i}{m}\right)t} \\
    % = & \sum_{t \in T} \E{T \sim \mathcal{N}(0, I)^m}{ \left(\alpha_i \cdot \frac{\exp{\ip{p, t}}}{Z_{p, i}} + \frac{1-\alpha_i}{m}\right)t} \\
    = & \sum_{t \in T} \E{T \sim \mathcal{N}(0, I)^m}{ \alpha_i \cdot \frac{\exp{\beta_i\ip{p, t}}}{Z_{p, i}} t} + \E{T \sim \mathcal{N}(0, I)^m}{ \frac{1-\alpha_i}{m} t}
  \end{align*}
  Note that the right term is just equal to zero since $t$ is sampled from the spherical Gaussian centered at the origin. The left term only depends on a single $t$ so we can change the expectation to be over a single $t$ sampled from a spherical Gaussian and replace the sum with a factor $m$.
  \begin{align*}
    = m \E{t \sim \mathcal{N}(0, I)}{ \alpha_i \cdot \frac{\exp{\beta_i\ip{p, t}}}{Z_{p, i}} t}
    = \frac{m \alpha_i}{Z_{p, i}} \E{t \sim \mathcal{N}(0, I)}{t\cdot\exp({\beta_i\ip{p, t}})}
  \end{align*}
  We take the orthogonal decomposition of the vector $t$ to be $t_{\parallel}$ and $t_{\perp}$ which are in directions parallel and perpendicular to $p$ respectively. Note that this allows us to set $\ip{p, t_{\perp}}$ to 0.
  \begin{align*}
    & = \frac{m \alpha_i}{Z_{p, i}} \E{t \sim \mathcal{N}(0, I)}{(t_{\parallel} + t_{\perp})\cdot\exp{\beta_i\ip{p, (t_{\parallel} + t_{\perp})}}} \\
    % & = \frac{m \alpha_i}{Z_{p, i}} \E{t \sim \mathcal{N}(0, I)}{(t_{\parallel} + t_{\perp})\cdot\exp{\left(\ip{p, t_{\parallel}}+ \ip{p, t_{\perp}}\right)}} \\
    & = \frac{m \alpha_i}{Z_{p, i}}\bigg( \E{t \sim \mathcal{N}(0, I)}{t_{\perp}\cdot\exp{\left(\beta_i\ip{p, t_{\parallel}}\right)}} + \\ & \hspace{2cm}\E{t \sim \mathcal{N}(0, I)}{t_{\parallel}\cdot\exp{\left(\beta_i\ip{p, t_{\parallel}}\right)}} \bigg)\\
    & = \frac{m \alpha_i}{Z_{p, i}}\bigg( \E{t \sim \mathcal{N}(0, I)}{t_{\perp}} \E{t \sim \mathcal{N}(0, I)}{\exp{\left(\beta_i\ip{p, t_{\parallel}}\right)}} + \\& \hspace{2cm} \E{t \sim \mathcal{N}(0, I)}{t_{\parallel}\cdot\exp{\left(\beta_i\ip{p, t_{\parallel}}\right)}} \bigg)
  \end{align*}
  We can factor the expectation in the last line since $t_{\parallel}$ is a random variable independent of $t_{\perp}$. Since $t_{\perp}$ is a linear transformation of $t$ and therefore, it is also a mean-zero Gaussian. Therefore, the first term in the sum goes to zero. Since $t_{\parallel}$ is a rank-$1$ linear transformation of $t$, we can compute it as a one-dimensional Gaussian with mean $0$ and variance $\|p\|_2^2 = 1$.
  % Link to Wolframalpha doing the integration
  % https://www.wolframalpha.com/input/?i=integrate+x+*+exp(x-x%5E2%2F2),+x%3D+-inf+to+inf
 \begin{align*}
     & = \frac{m \alpha_i p}{Z_{p, i}} \E{x \sim \mathcal{N}(0, 1)}{x\cdot\exp{\left(\beta_ix\right)}}\\
     % & = \frac{m \alpha_i p}{Z_{p, i}\sqrt{2 \pi\|p\|_2^2}} \int_{- \infty}^{\infty}{x\cdot\exp{\left(\beta_ix - \frac{x^2}{2\|p\|_2^2}\right)}} ~ dx\\
     % & = \frac{m \alpha_i p}{Z_{p, i}\sqrt{2 \pi}} \int_{- \infty}^{\infty}{x\cdot\exp{\left(\beta_ix - \frac{x^2}{2}\right)}} ~ dx\\
     % & = \frac{m \alpha_i p}{Z_{p, i}\sqrt{2 \pi}} \sqrt{2\pi} \beta_i\exp(\beta_i^2/2) \\
     & = \frac{m\alpha_i\beta_i\exp(\beta_i^2/2)}{Z_{p, i}} \cdot p
 \end{align*}
 \end{proof}

 \begin{lemma}\label{appendix:lem:variance-of-trigram}
The expected $\ell_2$-distance squared of a sampled trigram from its mean is $\Theta(d \ln d) - \rho_i^2$
\end{lemma}
\begin{proof}\label{appendix:proof:variance-of-trigram}
Without loss of generality we assume that the product vector is the $(1, 0, \dots, 0)$ vector.
  \begin{align*}
    &\E{T \sim \mathcal{N}(0, I)^m}{\E{t \sim P_p}{\|t - \rho_i p\|^2}}\\
    &= \E{T \sim \mathcal{N}(0, I)^m}{\E{t \sim P_p}{(t_1 - \rho_i)^2 + t_2^2 + \dots + t_d^2}}\\
    &= \E{T \sim \mathcal{N}(0, I)^m}{\E{t \sim P_p}{t_1^2 + \rho_i^2 - 2\rho_i t_1 + t_2^2 + \dots + t_d^2}}\\
    &= \rho_i^2 - 2\rho_i\E{T \sim \mathcal{N}(0, I)^m}{\E{t \sim P_p}{t_1}} + \E{T \sim \mathcal{N}(0, I)^m}{\E{t \sim P_p}{\|t\|^2}}\\
      % & = \E{T \sim \mathcal{N}(0, I)^m}{\E{t \sim P_p}{\|t\|^2 + \rho_i^2 \|p\|^2 - 2 \rho_i \ip{t, p}}} \\
      % & = \rho_i^2 + \E{T}{\sum_{t \in T} P_p(t) \|t\|^2} - 2 \rho_i \E{T}{\sum_{t \in T} P_p(t) \ip{t, p}} \\
      % & = \rho_i^2 + \E{T}{\sum_{t \in T} \left( \alpha \frac{\exp(\ip{t, p})}{Z_{p, i}} \|t\|^2 + \frac{(1 - \alpha)}{m} \|t\|^2 \right)} - 2 \rho_i \E{T}{\sum_{t \in T} \left( \alpha \frac{\exp(\ip{t,p})}{Z_{p, i}} \ip{t,p} + \frac{(1 -  \alpha)}{m} \ip{t, p} \right)}
  \end{align*}

  Note that the first expectation term can be dealt with easily enough.
  \begin{align*}
    &\E{T \sim \mathcal{N}(0, I)^m}{\E{t \sim P_p}{t_1}}\\
    &= \E{T \sim \mathcal{N}(0, I)^m}{\sum_{t\in T} t_1 \left(\alpha_i \cdot \frac{\exp{\beta_i\ip{p, t}}}{Z_{p, i}} + \frac{1-\alpha_i}{m}\right)}\\
    &= \frac{\alpha_i }{Z_{p, i}}\sum_{t\in T}\E{T \sim \mathcal{N}(0, I)^m}{ t_1 \cdot \exp{\beta_it_1}} + \E{T \sim \mathcal{N}(0, I)^m}{t_1\frac{1-\alpha_i}{m}}\\
  \end{align*}
  Now the second term is zero since the expectation is simply over a standard normal random variable. Also note that the sum in the first term can be replaced with a factor of $m$ since each term in the sum has the same value. The first expectation is over a standard normal variable and it can be explicitly calculated just like in the proof of Lemma~\ref{appendix:lem:expectation-of-trigram} to show that it is equal to $\rho_i$.
  \begin{align*}
    \E{T \sim \mathcal{N}(0, I)^m}{\E{t \sim P_p}{t_1}}
    = \frac{\alpha_i m}{Z_{p, i}}\E{x \sim \mathcal{N}(0, 1)}{ x \cdot \exp{\beta_ix}} \;=\; \frac{\alpha_i m \beta_i \exp{\beta_i^2/2}}{Z_{p, i}} = \rho_i
  \end{align*}

  That only leaves the expectation of the squared norm of a trigram left.

\begin{align*}
    \E{T \sim \mathcal{N}(0, I)^m}{\E{t \sim P_p}{\|t\|^2}} & \leq  \E{T \sim \mathcal{N}(0, I)^m}{{\|t\|^2} ~|~ t = \argmax_{t' \in T} \|t'\|^2} \\
    & \approx \E{X \sim \mathcal{G}}{X}
\end{align*}

    where the last line follows by the well-known approximation of $\chi^2$ block maxima by the Gumbel distribution $\mathcal{G}$ as $m \rightarrow \infty$ (\cite{embrechts1997modeling}, 4, p.156). The mean of the Gumbel distribution is:
    $$ \frac{\gamma}{2} + 2\left(\ln m + \left(\frac{d}{2}-1\right) \ln \ln m - \ln \left(\Gamma\left(\frac{d}{2}\right)\right)\right) = \Theta(d \ln d)$$
    where $\gamma$ is the Euler-Mascheroni constant, and the equality follows from Stirling's approximation for constant $m$. \cite{gasull2015maxima} suggests nonstandard approximation terms which converge much faster to the Gumbel distribution for our parameter regime.

  Combining all of the above we get that the expected $\ell_2$ distance squared of a trigram from its mean is: $\Theta(d) - \rho_i^2$.
\end{proof}

\begin{lemma}[Concentration of partition functions, Lemma 2.1 from~\cite{arora16randwalk}]
  \label{appendix:lem:concentration-of-Z}
  For trigrams vectors of the form $v_t = s_t \cdot \hat{v}_t$, where $\hat{v}_t$ comes from a spherical Gaussian distribution, and for $\eps_z = \widetilde{O}(1/\sqrt{m})$ and $\delta = \exp(-\Omega(\log^2 m))$ there exists $Z_i$ s. t.,
    $\prOther{p\sim S}{(1-\eps_z)Z_i \leq Z_{p, i} \leq (1+\eps_z) Z_i} \geq 1 - \delta$.
\end{lemma}

\subsection{Low rank of PMI}
\begin{theorem}\label{appendix:thm:co-occurrence-probability}
Let $q, q'$ be query vectors generated by the AttEST model. Denote the probability that $q, q'$ co-occur in the query graph by $\pr{q, q'}$. Then,
\begin{align*}
    \pr{q, q'} = & ~ (1\pm \eps') \cdot \Bigg(\prod_{i \in [n(q)]} \frac{\alpha_{i}}{Z_i}\prod_{j \in [n(q')]} \frac{\alpha_{j}}{Z_j} \Bigg) \cdot \exp\Bigg(\frac{\norm{\sum_{i\in[n(q)]}\beta_i t_i + \sum_{j\in[n(q')]}\beta_j t_j}^2}{2d}\Bigg)
\end{align*}
\end{theorem}

\begin{theorem}\label{appendix:thm:single-co-occurrence-probability}
Let $q$ be a query vector generated by the AttEST model. Then,
\begin{align*}
      \pr{q} &= (1\pm \eps'') \cdot \Bigg(\prod_{i \in [n(q)]} \frac{\alpha_{i}}{Z_i}\Bigg)\cdot \exp\Bigg(\frac{\norm{\sum_{i\in[n(q)]}\beta_i t_i}^2}{2d}\Bigg)\\
\end{align*}
\end{theorem}
The proof of Theorem~\ref{appendix:thm:single-co-occurrence-probability} follows a similar argument to Theorem~\ref{appendix:thm:co-occurrence-probability} and together they imply:
\begin{corollary}\label{appendix:cor:pmi-dot-product}
For $\eps''' = O(\eps' + \eps'')$,
$\textrm{PMI}(q, q') \triangleq \log \Bigg( \frac{\Pr[q, q']}{\Pr[q]\Pr[q']} \Bigg) = \frac{\ip{q, q'}}{d} + \eps'''.$
\end{corollary}

To prove Theorem \ref{appendix:thm:co-occurrence-probability}, we make use of the following lemmas.

\begin{lemma}[Length of spherical Gaussian concentrates]\label{appendix:lem:spherical-gaussian-norm-concentrates}
  Let $x$ be a spherical Gaussian vector in $d$ dimensions, then
  \begin{align*}
    \pr{\|x\| \geq b\sqrt{d}}  &\leq \exp{(-b'\cdot d)}\\
    \pr{\|x\| \leq \sqrt{d}/c}  &\leq \exp{(-d/c')}
\end{align*}
\end{lemma}
\begin{proof}
The above lemma is simply a reparametrization of the corollary of Lemma 1 from \cite{laurent-massart-2000-concentration-of-chi-2}. For the first bound,
  \begin{align*}
    \pr{\|t\|^2  - d  \geq 2 \sqrt{dx} + 2x}  \leq \exp{(-x)} \\\implies \pr{\|t\| \geq b\sqrt{d}}  \leq \exp{(-b'\cdot d)}
\end{align*}
One can solve $b'= 2b^2+2b+1$ to get the exact relation between $b$ and $b'$. % A similar argument gets us the other bound.
For the second bound,
  \begin{align*}
    \pr{d - \|t\|^2  \geq 2 \sqrt{dx}}  \leq \exp{(-x)} \\\implies \pr{\|t\| \geq b\sqrt{d}}  \leq \exp{(-b'\cdot d)}
\end{align*}
One can solve $c'= 2c^2+2c+1$ to get the exact relation between $c$ and $c'$.
\end{proof}

\begin{lemma}\label{appendix:lem:approximate-partitions-as-Z-i}
We can remove the dependence of the partition function on the products to approximate the following
  \begin{align*}
    \prod_{i \in [n(q)]} \alpha_{i} \frac{\exp\big(\beta_i\ip{t_i, p}\big)}{Z_{p, i}}  + \frac{(1 - \alpha_{i})}{m} \quad\cdot\prod_{j \in [n(q')]} \alpha_{j}\frac{\exp\big(\beta_j\ip{t'_j, p'}\big)}{Z_{p', j}}  + \frac{(1 - \alpha_{j})}{m}
  \end{align*}
  with
  \begin{align*}
    (1+O(\eps_z))^2\prod_{i \in [n(q)]} \alpha_{i} \frac{\exp\big(\beta_i\ip{t_i, p}\big)}{Z_i}  + \frac{(1 - \alpha_{i})}{m} \quad\cdot\prod_{j \in [n(q')]} \alpha_{j}\frac{\exp\big(\beta_j\ip{t'_j, p'}\big)}{Z_{j}}  + \frac{(1 - \alpha_{j})}{m}
  \end{align*}
  where $\eps_z=\widetilde{O}(1/\sqrt{m})$.
\end{lemma}
\begin{proof}
  Lemma~\ref{appendix:lem:concentration-of-Z} implies that with probability at least $1 - \delta$, the following holds:
  \begin{align*}
    % \prod_{i \in [n(q)]} \frac{\alpha_{i}}{(1+\epsilon_z)} \cdot\frac{\exp\big(\beta_i\ip{t_i, p}\big)}{Z_{i}}  + \frac{(1 - \alpha_{i})}{(1+\epsilon_z)m} \quad\cdot\prod_{j \in [n(q')]} \frac{\alpha_{j}}{(1+\epsilon_z)} \cdot\frac{\exp\big(\beta_j\ip{t'_j, p'}\big)}{Z_{j}}  + \frac{(1 - \alpha_{j})}{(1+\epsilon_z)m}
    % \leq\\
    % \prod_{i \in [n(q)]} \alpha_{i} \frac{\exp\big(\beta_i\ip{t_i, p}\big)}{Z_{p, i}}  + \frac{(1 - \alpha_{i})}{m} \quad\cdot\prod_{j \in [n(q')]} \alpha_{j}\frac{\exp\big(\beta_j\ip{t'_j, p'}\big)}{Z_{p', j}}  + \frac{(1 - \alpha_{j})}{m}
    % \leq\\
    % \prod_{i \in [n(q)]} \frac{\alpha_{i}}{(1-\epsilon_z)} \cdot\frac{\exp\big(\beta_i\ip{t_i, p}\big)}{Z_{i}}  + \frac{(1 - \alpha_{i})}{(1-\epsilon_z)m} \quad\cdot\prod_{j \in [n(q')]} \frac{\alpha_{j}}{(1-\epsilon_z)} \cdot\frac{\exp\big(\beta_j\ip{t'_j, p'}\big)}{Z_{j}}  + \frac{(1 - \alpha_{j})}{(1-\epsilon_z)m}\\
   \frac{1}{(1+\epsilon_z)^{2n(q)}} \cdot \prod_{i \in [n(q)]} \alpha_i\frac{\exp\big(\beta_i\ip{t_i, p}\big)}{Z_{i}}  + \frac{(1 - \alpha_{i})}{m} \quad\cdot \prod_{j \in [n(q')]} \alpha_{j}\frac{\exp\big(\beta_j\ip{t'_j, p'}\big)}{Z_{j}}  + \frac{(1 - \alpha_{j})}{m}\\
    \leq
    \prod_{i \in [n(q)]} \alpha_{i} \frac{\exp\big(\beta_i\ip{t_i, p}\big)}{Z_{p, i}}  + \frac{(1 - \alpha_{i})}{m} \quad\cdot\prod_{j \in [n(q')]} \alpha_{j}\frac{\exp\big(\beta_j\ip{t'_j, p'}\big)}{Z_{p', j}}  + \frac{(1 - \alpha_{j})}{m}\\
    \leq
        \frac{1}{(1-\epsilon_z)^{2n(q')}} \prod_{i \in [n(q)]} \alpha_j\frac{\exp\big(\beta_i\ip{t_i, p}\big)}{Z_{i}}  + \frac{(1 - \alpha_{i})}{m} \quad\cdot \prod_{j \in [n(q')]} \alpha_{j}\frac{\exp\big(\beta_j\ip{t'_j, p'}\big)}{Z_{j}}  + \frac{(1 - \alpha_{j})}{m}
  \end{align*}
Using the approximation $1 + x \approx e^x$ for small $x$, we can replace the $(1 + \epsilon_z)^{2n(q)} \approx \exp(2n(q)) \approx (1 + 2n(q)\epsilon_z)$ and similarly for the other term completing the proof.
\end{proof}

This can be proven using Lemma~\ref{appendix:lem:concentration-of-Z} following an argument similar to the one in the proof of Theorem 2.1 in \cite{arora16randwalk}.

\begin{lemma}[Ignoring the uniform component]\label{appendix:lem:ignoring-uniform-component}
  If $\beta_i \geq \frac{1}{b\sqrt{d}}\cdot \log \bigg(\frac{\alpha_i}{1-\alpha_i} \cdot \frac{m}{Z_i} \cdot\frac{\eps_u}{a}\bigg)$ then with high probability the uniform component can be ignored without incurring too much error. Formally, let  $A=\alpha_{i} \frac{\exp\big(\beta_i\ip{t_i, p}\big)}{Z_i}$ be the exponential component, $B=\frac{(1 - \alpha_{i})}{m}$ be the uniform component then the following holds:
  \begin{align*}
    \pr{(1-\eps_u)(A+B)\leq A \leq (1+\eps_u)(A+B)} = 1 - e^{-\Theta(d)}
    % \pr{\alpha_{i} \frac{\exp\big(\beta_i\ip{t_i, p}\big)}{Z_i} \neq (1\pm \eps_u)~\bigg(\alpha_{i} \frac{\exp\big(\beta_i\ip{t_i, p}\big)}{Z_i}  + \frac{(1 - \alpha_{i})}{m}\bigg)} = e^{-O(x^2)}
  \end{align*}
  We denote by $\FF_u$ the event where the uniform component can be ignored. % e, so that $\pr{\FF_u} = 1 - e^{-\Theta(d)}$.
\end{lemma}

\begin{lemma}[Lemma A.5 from \cite{arora16randwalk}]\label{appendix:lem:ip-to-norm}
  Let $v\in \RR^d$ be a fixed vector with norm $\|v\| \le \kappa \sqrt{d}$ for absolute constant $\kappa$. Then for random variable $c$ with uniform distribution over the sphere, we have that
  \begin{equation}
    \log \E{}{\exp(\ip{v, d})} = \|v\|^2/{2d} \pm \epsilon_n
  \end{equation}
  where $\epsilon_n = \widetilde{O}(\frac{1}{d})$.
\end{lemma}
\begin{proof}[Proof of Theorem~\ref{appendix:thm:co-occurrence-probability}]\label{appendix:proof:co-occurence-probability}
Our proof for the co-occurrence probability initially follows the structure of the proof of Theorem 2.2 in \cite{arora16randwalk}, particularly for the concentration of the partition functions. However, after that point our proof crucially diverges to deal with the mixture distributions of multiple trigrams. First, using the law of total expectation, we can write the probability of co-occurrence in terms of the probability of sampling by product vectors that are within distance $\eps$ of each other.
\begin{align*}
  \pr{q, q'} & = \E{p, p' \sim S}{\pr{q, q'} \;~\big|~\; \|p - p'\| \leq \eps_p/d)}\\
             & = \E{p \sim S}{\pr{q, q'} \;~\big|~\; \|p - p'\| \leq \eps_p/d)}\\
             & = \E{p \sim S}{\pr{q~|~p} \pr{q' ~|~ p'} } \\
            & = \E{p \sim S}{\prod_{i \in [n(q)]} \pr{t_i ~|~ p} \prod_{j \in [n(q')]} \pr{t'_j ~|~ p'}} \\
            & = \mathop{\mathbb{E}}_{p \sim S}\Bigg[\prod_{i \in [n(q)]} \alpha_{i} \frac{\exp\big(\beta_i\ip{t_i, p}\big)}{Z_{p, i}}+ \frac{(1 - \alpha_{i})}{m}\prod_{j \in [n(q')]} \alpha_{j}\frac{\exp\big(\beta_j\ip{t'_j, p'}\big)}{Z_{p', j}} + \frac{(1 - \alpha_{j})}{m}\Bigg]
\end{align*}
The second step is valid since the only property of $p'$ that we will use is that is it in an $\ell_2$-ball of radius $\eps$ around $p$. We now use Lemma~\ref{appendix:lem:approximate-partitions-as-Z-i} to remove the dependence of the partition functions on the product vectors $p$ and $p'$ to get the following:
% Using Lemma~\ref{lem:concentration-of-Z} we can assume that the partition functions $Z_{p, i}$ and $Z_{p', i}$ concentrate at $Z_i$ with high probability. Following an argument similar to \cite{arora16randwalk}, we can thus replace all occurrences of the partition functions with the constants $Z_i$ while incurring a small error factor.
% More formally, we condition on the event $F$ that the partition functions concentrate. When $F$ happens, the error incurred by approximating the partition functions as $Z$ is small and the co-occurrence probability is still within $O(1\pm \eps_F)$ for a small $\eps_F$. On the other hand, the probability of $F$ not happening is very small and the contribution of that term to the co-occurrence probability can be ignored while still being within $O(1\pm \eps_F)$.

\begin{align*}
    \pr{q, q'} & = (1+O(\eps_z))^2\cdot \mathop{\mathbb{E}}_{p \sim S}\Bigg[\prod_{i \in [n(q)]} \alpha_{i} \frac{\exp\big(\beta_i\ip{t_i, p}\big)}{Z_i}  \\&\quad+ \frac{(1 - \alpha_{i})}{m}\prod_{j \in [n(q')]} \alpha_{j}\frac{\exp\big(\beta_j\ip{t'_j, p'}\big)}{Z_j}  \quad+ \frac{(1 - \alpha_{j})}{m}\Bigg]
\end{align*}

Let $G$ be the term inside the expectation above, and $G'$ be the term without the uniform components. We condition on the event $\FF_u$ from Lemma~\ref{appendix:lem:ignoring-uniform-component}. When $\FF_u$ happens Lemma~\ref{appendix:lem:ignoring-uniform-component} allows us to ignore the uniform component.
\begin{align*}
  \E{p \sim S}{G} & = \E{p \sim S}{G~|~\FF_u}\E{}{\FF_u} + \E{p \sim S}{G~|~\overline{\FF_u}}\E{}{\overline{\FF_u}}\\
                  & = \E{p \sim S}{G~|~\FF_u}\pr{\FF_u} + \E{p \sim S}{G~|~\overline{\FF_u}}\pr{\overline{\FF_u}}\\
                  & = \E{p \sim S}{G~|~\FF_u}(1-\exp{(-\Theta(d))}) \\&\quad+ \E{p \sim S}{G~|~\overline{\FF_u}}(\exp{(-\Theta(d))})\\
  \implies \E{p \sim S}{G} & = \frac{(1-2\exp{(-\Theta(d))})}{(1-\exp{(-\Theta(d))})}\cdot\E{p \sim S}{G~|~\FF_u}\\&\quad \approx \E{p \sim S}{G|\FF_u} = \E{p \sim S}{G'}
\end{align*}
Some algebraic manipulations allow us to show that conditioning on $\FF_u$ allows us to focus on the $\E{p \sim S}{G'}$ term alone. Note that for practical purposes $\frac{(1-2\exp{(-\Theta(d))})}{(1-\exp{(-\Theta(d))})}$ is equal to 1 at values of $d\geq 300$. From here we get that the probability of co-occurence is:
\begin{align*}
  &(1+O(\eps_z))^2\cdot (1\pm O(\eps_{xyz})) \cdot \quad  \E{p \sim S}{\prod_{i \in [n(q)]} \alpha_{p, i} \frac{\exp\big(\beta_i\ip{t_i, p}\big)}{Z_i}\prod_{j \in [n(q)]} \alpha_{p', j}\frac{\exp\big(\beta_j\ip{t'_j, p'}\big)}{Z_{j}}} 
             \\& = (1+O(\eps_z))^2\cdot (1\pm O(\eps_{xyz})) \cdot \\& \quad  \prod_{i\in [n(q)]}\frac{\alpha_i}{Z_i}\cdot \prod_{j\in [n(q')]}\frac{\alpha_j}{Z_j}\cdot \mathop{\mathbb{E}}_{p \sim S}\Bigg[\exp\Big(\ip{\sum_{i \in [n(q)]} \beta_i t_i, p} \quad  + \ip{\sum_{j \in [n(q')]} \beta_j t'_j, p'}\Big)\Bigg]
\end{align*}
At this point we use the fact that $\|p-p'\|\leq \eps$ to approximate $p'$ as $p$.
\begin{align*}
  &{\exp\Big(\ip{\sum_{i \in [n(q)]} \beta_i t_i, p} + \ip{\sum_{j \in [n(q')]} \beta_j t'_j, p'}\Big)}\\
  &= {\exp\Big(\ip{\sum_{i \in [n(q)]} \beta_i t_i + \sum_{j \in [n(q')]} \beta_j t'_j, p}\Big) \exp\Big(\ip{\sum_{j \in [n(q')]} \beta_j t'_j, p' - p}\Big)}\\
  &\leq {\exp\Big(\ip{\sum_{i \in [n(q)]} \beta_i t_i + \sum_{j \in [n(q')]} \beta_j t'_j, p}\Big) \exp\Big(\|\sum_{j \in [n(q')]} \beta_j t'_j\|\cdot \|p' - p\|\Big)}\\
  &\leq {\exp\Big(\ip{\sum_{i \in [n(q)]} \beta_i t_i + \sum_{j \in [n(q')]} \beta_j t'_j, p}\Big) \exp\Big(\eps/d \cdot \sum_{j \in [n(q')]} \beta_j\| t'_j\| \Big)}\\
  \end{align*}
  Now we can use a similar argument as in the proof of Lemma~\ref{appendix:lem:variance-of-trigram} to show that the lengths of each trigram vector is $\Theta(\sqrt{d})$ with high probability. This allows us to say that with high probability, the second exponent is of the form $\exp(O(\eps_p))$. Substituting that in the overall expression we get,
  \begin{align*}
  \pr{q, q'} & = (1+O(\eps_z))^2\cdot (1\pm O(\eps_{xyz})) \cdot (1+O(\eps_{p})) \\&\quad  \cdot \prod_{i\in [n(q)]}\frac{\alpha_i}{Z_i}\cdot \prod_{j\in [n(q')]}\frac{\alpha_j}{Z_j}\cdot \mathop{\mathbb{E}}_{p \sim S}\Bigg[\exp\Big(\ip{\sum_{i \in [n(q)]} \beta_i t_i \\&\quad + \sum_{j \in [n(q')]} \beta_j t'_j, p}\Big)\Bigg]
\end{align*}

% Then we condition on the whether the partition function concentrates. We can show that the probability of this not happening is low and hence can be ignored without incurring too much error.
% \begin{align*}
%     & = \frac{(1+O(\eps_2))}{Z^{n(q)+n(q')}}\E{p \sim S}{\prod_{i \in [n(q)]} \alpha_{p, i} \exp\big(\beta_i\ip{t_i, p}\big)\prod_{j \in [n(q)]} \alpha_{p', j}\exp\big(\beta_j\ip{t'_j, p'}\big)} \\
% \end{align*}
% Since $p$ and $p'$ are $\eps$ close, we can rearrange the inner products in terms of $p$ and $p-p'$. The second type of inner product will be small since $\|p-p'\|$ is small. Some algebraic manipulations then give us the following expression:
% \begin{align*}
%     & = \frac{(1+O(\eps_3))}{Z^{n(q)+n(q')}} \cdot \Bigg(\prod_{i, j \in [n(q)]} \alpha_{p, i}\alpha_{p', j} \Bigg)\cdot \E{p \sim S}{\exp\Big(\ip{\sum_{i\in[n(q)]}\beta_i t_i + \sum_{j\in[n(q)]}\beta_j t'_j, p}\Big)}\\
% \end{align*}
Using Lemma~\ref{appendix:lem:ip-to-norm}, we can replace the expectation of the exponential of inner product with the exponential of the norm. This gives us the final form of the probability of co-occurrence.
\begin{align*}
  & (1\pm \eps') \prod_{i\in [n(q)]}\frac{\alpha_i}{Z_i}\cdot \prod_{j\in [n(q')]}\frac{\alpha_j}{Z_j}\\&\quad \cdot \exp\Bigg(\frac{\big\|\sum_{i \in [n(q)]} \beta_i t_i + \sum_{j \in [n(q')]} \beta_j t'_j\big\|^2}{2d}\Bigg)
\end{align*}
where $1- \eps' \leq  (1+O(\eps_z))^2\cdot (1\pm O(\eps_{xyz})) \cdot (1+O(\eps_{p})) \cdot (1\pm O(\eps_n)) \leq 1+\eps'$
% \begin{align*}
%     & = \frac{(1+O(\eps_4))}{Z^{n(q)+n(q')}} \cdot \Bigg(\prod_{i, j \in [n(q)]} \alpha_{p, i}\alpha_{p', j} \Bigg)\cdot \exp\Bigg(\frac{\norm{\sum_{i\in[n(q)]}\beta_i t_i + \sum_{j\in[n(q)]}\beta_j t'_j}^2}{2d}\Bigg)\\
% \end{align*}
\end{proof}
Now we prove Lemma~\ref{appendix:lem:ignoring-uniform-component}.
\begin{proof}[Proof of Lemma~\ref{appendix:lem:ignoring-uniform-component}]\label{appendix:proof:ignoring-uniform-component}
  Let $A = \alpha_{i} \frac{\exp\big(\beta_i\ip{t_i, p}\big)}{Z_i}$ and let $B = \frac{1-\alpha}{m}$. Since $B\geq 0$ the upper bound $A \leq A+B \leq (1+\eps_u) (A+B)$ always holds. It only remains to show that $A\geq (1-\eps_u)(A+B)$, and it suffices to show that $\eps_z A \geq B$. Since $\alpha \in [\nicefrac{1}{2}, 1]$ and the constant $Z_i = \Theta(m)$ for all $i$, we can further simplify the inequality to the following:
  \begin{align*}
    & \eps_u \alpha_{i} \frac{\exp\big(\beta_i\ip{t_i, p}\big)}{Z_i} \geq \frac{(1 - \alpha_{i})}{m}\\
    % \iff & \beta_i \geq \frac{1}{\ip{t_i, p}} \cdot \log \big(\frac{1-\alpha_i}{\alpha_i} \cdot \frac{Z_i}{m} \cdot \frac{1}{\eps_u}\big)\\
    % \iff & \beta_i \geq \frac{1}{\ip{t_i, p}} \cdot \Omega\bigg(\log \bigg(\frac{1}{\eps_u}\bigg)\bigg)\\
  \end{align*}
  Substituting the lower bound for $\beta_i$ into the LHS of the above inequality, we get: % Starting with our assumption regarding the lower bound of $\beta$, we proceed to show the above inequality.
  \begin{align*}
    \beta_i &\geq \frac{1}{b\sqrt{d}}\cdot \log \bigg(\frac{\alpha_i}{1-\alpha_i} \cdot \frac{m}{Z_i} \cdot\frac{\eps_u}{a}\bigg) \\
    &\implies \eps_u \alpha_{i} \frac{\exp\big(\beta_i\ip{t_i, p}\big)}{Z_i} \\&\quad \geq \eps_u \frac{\alpha_{i}}{Z_i} \exp\Bigg(\frac{\ip{t_i, p}}{b\sqrt{d}}\cdot \log \bigg(\frac{\alpha_i}{1-\alpha_i} \cdot \frac{m}{Z_i} \cdot\frac{\eps_u}{a}\bigg)\Bigg) \\
            & \geq \eps_u \frac{\alpha_{i}}{Z_i} \exp\Bigg(\frac{-\|t_i\|}{b\sqrt{d}}\cdot \log \bigg(\frac{\alpha_i}{1-\alpha_i} \cdot \frac{m}{Z_i} \cdot\frac{\eps_u}{a}\bigg)\Bigg) \\
            & \geq \eps_u \frac{\alpha_{i}}{Z_i} \exp\Bigg(-1 \cdot \log \bigg(\frac{\alpha_i}{1-\alpha_i} \cdot \frac{m}{Z_i} \cdot\frac{\eps_u}{a}\bigg)\Bigg) \\
            &\geq a \frac{1-\alpha_i}{m}
  \end{align*}
  We use the fact that the dot product is smallest when the vectors are in opposite directions. Then we use Lemma~\ref{appendix:lem:spherical-gaussian-norm-concentrates} to show that the norm of $\|t_i\|$ is upper bounded with probability $1 - \exp{(-\Theta(d))}$. % and then we use the high probability event that the norm of a Spherical Gaussian vector is upper bounded with high probability using the corollary of Lemma 1 from \cite{laurent-massart-2000-concentration-of-chi-2}.
%   \begin{align*}
%     \pr{\|t\|^2  - d  \geq 2 \sqrt{dx} + 2x}  \leq \exp{(-x)} \implies \pr{\|t\| \geq b\sqrt{d}}  \leq \exp{(-b'\cdot d)}
% \end{align*}
% One can solve $b'= 2b^2+2b+1$ to get the exact bound, we only care about the probability being $\exp{(-\Theta(d))}$ which proves our lemma.
\end{proof}

\subsection{Loss function derivation}

To derive the loss function, we start by maximizing the mutual information (MI) of the marginal distributions of each endpoint of the edge distribution $E$ while minimizing the MI between the marginal distribution of each endpoint of the non-edge distribution $\bar{E}$. This ensures that the amount of information derived from a query embedding about its related queries (and only its related queries) is maximized. Let $E_1, E_2$ be the marginal distributions for the first and second vertex in the edge sampled from $E$. Similarly, let $\bar{E}_1, \bar{E}_2$ be the marginal endpoint distributions for non-edges sampled from $\bar{E}$. Recall the definition of mutual information,
\begin{align*}
I(E_1; E_2) \triangleq & \sum_{q, q' \in Q} \Pr[q, q']~ \textrm{PMI}(q, q').
\end{align*}
We maximize the mutual information between $E_1$ and $E_2$ while minimizing it between $\bar{E}_1$ and $\bar{E}_2$.
\begin{align*}
    \argmax I(E_1; E_2) - I(\bar{E}_1; \bar{E}_2)
\end{align*}
We can equivalently maximize the exponential of the above.
\begin{align*}
    =& ~ \argmax ~\exp(I(E_1; E_2) - I(\bar{E}_1; \bar{E}_2))\\
    =& ~ \argmax ~\prod_{q, q' \in Q} \exp(\pr{q, q'}\textrm{PMI}(q, q') - \pr{q, q' \text{ not adjacent}]}\textrm{PMI}(q, q'))
    \intertext{For the first term, we can re-index the product over the queries $P(q)$ adjacent to $q$ since otherwise the empirical probability that the queries co-occur is $0$. For the second term, we can similarly re-index the product over the non-edges $N(q)$.}
    =& ~ \argmax ~\prod_{q \in Q, q' \in P(q)} \exp(\pr{q, q'}\textrm{PMI}(q, q')) \prod_{q, q' \in Q} \exp(-\pr{q, q' \text{ not adjacent}]}\textrm{PMI}(q, q'))\\
    % =& \argmax \prod_{q, q' \in D} (\frac{\pr{q, q'}}{\pr{q}\pr{q'}}})^{pr{q, q'}}\\
    % =& \argmax \prod_{q, q' \in D} (\frac{\pr{q, q'}}{\pr{q}\pr{q'}}})^{pr{q, q'}}\\
      =& ~ \argmax ~\prod_{q \in Q, q' \in P(q)} \exp(\textrm{PMI}(q, q'))^{\pr{q, q'}} \prod_{q \in Q, q' \in N(q)} \exp(-\textrm{PMI}(q, q'))^{\pr{q, q' \text{ not adjacent}]}}
      \intertext{The empirical value of $\pr{q, q'}$ is weight of the edge between $q, q'$ divided by the sum of all the edge weights. Let $W_{q, q'}$ be the weight of the edge between $q$ and $q'$ and let $W = \sum_{q, q'} W_{q, q'}$.}
    =& ~ \argmax ~\prod_{q \in Q, q' \in P(q)} \exp(\textrm{PMI}(q, q'))^{W_{q, q'}/W} \prod_{q \in Q, q' \in N(q)} \exp(-\textrm{PMI}(q, q'))^{\pr{q, q' \text{ not adjacent}]}}\\
    =& ~ \argmax ~\prod_{q \in Q, q' \in P(q)} \exp(\textrm{PMI}(q, q'))^{W_{q, q'}} \prod_{q \in Q, q' \in N(q)} \exp(-\textrm{PMI}(q, q'))^{\pr{q, q' \text{ not adjacent}]}}
    \intertext{In order to improve performance, we set the weight of all edges to be $1$. Since the vast majority of queries are unique, this is analogous to reducing the frequency of extremely common words in word embedding models.}
    \approx& ~ \argmax ~\prod_{q \in Q, q' \in P(q)}\exp (\textrm{PMI}(q, q')) \prod_{q \in Q, q' \in N(q)} \exp(-\textrm{PMI}(q, q'))^{\pr{q, q' \text{ not adjacent}]}}
    \intertext{For the second term, equalize the empirical probability that $q, q'$ are not adjacent to some negligible amount. This allows ignoring the exponent. }
    \approx& ~ \argmax ~\prod_{q \in Q, q' \in P(q)}\exp (\textrm{PMI}(q, q')) \prod_{q \in Q, q' \in N(q)} \exp(-\textrm{PMI}(q, q')) \\
    %=& \argmax ~\exp \left(\sum_{q \in Q, q' \in P(q)} \textrm{PMI}(q, q')\right)\\
    \intertext{Apply Corollary \ref{appendix:cor:pmi-dot-product}.}
    =& ~ \argmax ~\exp \left(\sum_{q \in Q, q' \in P(q)} \ip{v_q, v_{q'}}\right) \exp \left(\sum_{q \in Q, q' \in N(q)} -\ip{v_q, v_{q'}}\right) \\
     =& ~ \argmax ~\exp \left(\sum_{q \in Q, q' \in P(q)} \ip{v_q, v_{q'}} ~+ \sum_{q \in Q, q' \in N(q)} -\ip{v_q, v_{q'}}\right)
\end{align*}
 where the last line follows from Corollary \ref{appendix:cor:pmi-dot-product}. We can remove the exponential due to monotonicity. Note that in a small range around $0$, the sigmoid function may be approximated by an exponential allowing us to take the logarithm of the sigmoid for each term in the sums, and get the loss function.
 %In order to avoid exploding and vanishing gradients, we apply the logarithm of logistic function to smooth out the value of the inner product. Heuristically, applying a logarithm leaves the dynamics of local search algorithms (gradient descent) largely unchanged \cite{SV14}.
\vspace{-0.6em}
\begin{align*}\argmax ~\sum_{q \in Q, q' \in P(q)} \log (\sigma ( \ip{v_q, v_{q'}})) ~~~ +  \sum_{q \in Q, q' \in N(q)} -\log(\sigma(-\langle v_q, v_{q'} \rangle)) \end{align*}
% Thus, maximizing MI for adjacent queries and minimizing it for non-adjacent queries gives the loss function .

\subsection{Experimental validation via trigram variance}
\begin{lemma} \label{appendix:lem:blue}
Let $X_1, \dots, X_k$ be independent, real-valued random variables drawn from different distributions, such that all distributions have the same expectation $\mu$ but (potentially) different variances $\sigma_1, \dots, \sigma_k$, respectively. Then, the best linear unbiased estimator (BLUE) is $\sum_{i=1}^k w_iX_i$ where $w_i = \frac{1/\sigma_i^2}{\sum_{i=1}^k 1/\sigma_i^2}$.
\end{lemma}
\begin{proof}
Without loss of generality, assume that $\E{}{X_i} = 0$. Consider the (affine) linear combination $w_0 + \sum_{i=1}^k w_iX_i$. Observe that $\E{}{w_0 + \sum_{i=1}^k w_iX_i} = w_0$ since $\E{}{X_i} = 0$, for $1 \leq i \leq k$. Hence, for an unbiased estimator we can assume $w_0 = 0$. Furthermore, the variance of the linear combination is $\sum_{i=1}^k w_i^2\sigma_i^2$ and hence the sign of the $w_i$'s does not matter to the estimator. Thus, without loss of generality, for a BLUE we can restrict our attention to convex combinations of the form $ \sum_{i=1}^k w_iX_i$ such that $\sum_{i=1}^k w_i = 1$  and $0 \leq w_i \leq 1$, for $1 \leq i \leq k$.

We obtain the BLUE by minimizing the variance of the estimator, $\sum_1^n (w_i \sigma_i)^2$ where $\sum_1^n w_i = 1$. Using method of Lagrange Multipliers and solving we get $w_i = (1/\sigma_i^2) / \Sigma_1^n (1/\sigma_i^2)$ with the lowest variance being $1 / \Sigma_1^n (1/\sigma_i^2)$.
\end{proof}

\end{document}